\documentclass{ecai2025} 

\usepackage{latexsym}
\usepackage{amssymb}
\usepackage{amsmath}
\usepackage{amsthm}
\usepackage{booktabs}
\usepackage{enumitem}
\usepackage{graphicx}
\usepackage{color}

\usepackage{balance} 
\usepackage{graphicx}
\usepackage{xspace}
\usepackage{fontawesome5}
\usepackage{tcolorbox}
\tcbuselibrary{breakable}
\usepackage{amsmath}
\usepackage{enumitem} 
\usepackage{algorithm}
\usepackage{algorithmic}
\usepackage{stfloats}
\usepackage{float}
\usepackage{subcaption} 
\usepackage{amssymb}

\newtheorem{theorem}{Theorem}
\newtheorem{assumption}{Assumption}

\newcommand{\BibTeX}{B\kern-.05em{\sc i\kern-.025em b}\kern-.08em\TeX}

\begin{document}


\begin{frontmatter}


\paperid{2691} 


\title{Bi-level Mean Field: Dynamic \\ Grouping for  Large-Scale MARL}


\author[a]{\fnms{Yuxuan}~\snm{Zheng}}
\author[a]{\fnms{Yihe}~\snm{Zhou}}
\author[a]{\fnms{Feiyang}~\snm{Xu}} 
\author[a]{\fnms{Miingli}~\snm{Song}} 
\author[b,*]{\fnms{Shunyu}~\snm{Liu}} 

\address[A]{Zhejiang University}
\address[B]{Nanyang Technological University}


\begin{abstract}
Large-scale Multi-Agent Reinforcement Learning~(MARL) often suffers from the curse of dimensionality, as the exponential growth in agent interactions significantly increases computational complexity and impedes learning efficiency.
To mitigate this, existing efforts that rely on Mean Field~(MF) simplify the interaction landscape by approximating neighboring agents as a single \textit{mean} agent, thus reducing overall complexity to pairwise interactions. 
However, these MF methods inevitably fail to account for individual differences, leading to aggregation noise caused by inaccurate iterative updates during MF learning. 
In this paper, we propose a Bi-level Mean Field~(BMF) method to capture agent diversity with dynamic grouping in large-scale MARL, which can alleviate aggregation noise via bi-level interaction. 
Specifically, BMF introduces a dynamic group assignment~module, which employs a Variational AutoEncoder~(VAE) to learn the representations of agents, facilitating their dynamic grouping over time.
Furthermore, we propose a bi-level interaction module to model both inter- and intra-group interactions for effective neighboring aggregation.
Experiments across various tasks demonstrate that the proposed BMF yields results superior to the state-of-the-art methods.
\end{abstract}

\end{frontmatter}


\section{Introduction}

Multi-Agent Reinforcement Learning~(MARL) involves multiple autonomous agents operating within a shared environment, which has been widely applied in domains such as robotics~\cite{viksnin2019police,piardi2019arena}, autonomous driving~\cite{yan2024multi,yeh2024toward} and UAV trajectory design~\cite{zhou2024joint,fan2024multiagent}. In practical real-world scenarios, many tasks involve a large number of agents, which necessitates the use of large-scale MARL for efficient cooperation. However, unlike traditional MARL, large-scale settings, especially in complex scenarios, face the curse of dimensionality. This is due to the massive volume of interactions among agents~\cite{zhou2024decentralized,shike2023mix}, which  increases computational costs and hampers the learning process.

To remedy this problem, the mean field (MF) method offers a solution by approximating the interactions between agents using an averaging mechanism~\cite{yang2018mean}. Specifically, it treats the collective influence of all other agents as a single virtual \textit{mean} agent, simplifying each agent interaction into one between itself and this virtual entity. This significantly reduces the dimensionality of the interaction space. However, this averaging mechanism disregards individual differences between all agents, and the noise introduced by iterative approximations can degrade the cooperative performance~\cite{zhu2021survey,xu2021robustness}.

Recent advancements in MF methods can be categorized into two primary ways: (1) The first aims to enhance the accuracy of the virtual mean agent model, as demonstrated by GAT-MF~\cite{hao2023gat}, GAMFQ~\cite{yang2023partially}, and MFRAD~\cite{wu2022weighted}. These methods transition from a simple averaging scheme to a weighted version by leveraging graph-based techniques, which help reduce noise introduced by naive averaging. However, they introduce a significant computational burden, as they require pairwise computations between all agent pairs. This challenges the scalability of these methods in large-scale MARL. (2) The second seeks to represent agent heterogeneity by grouping agents into multiple types and assigning a virtual mean agent to each group, as seen in MTMF~\cite{ganapathi2020multi} and NPG~\cite{mondal2022approximation}. Compared to orignal MF, they improve accuracy by constructing multiple virtual mean agents. But these methods rely heavily on prior knowledge to define agent groups and do not consider the impact of actions from other groups. This lack of communication between groups leads to information loss during long-term aggregations, limiting adaptability to dynamic environments and different tasks.

In this paper, we introduce a novel Bi-level Mean Field (BMF)  method for large-scale MARL, which can alleviate interaction aggregation noise while maintaining low computational overhead, making it adaptable to diverse and dynamic agent-based tasks. The proposed BMF comprises two key components, namely, dynamic group assignment module and bi-level interaction module. 
Specifically, the dynamic group assignment module utilizes a VAE-based extractor to derive agent representations based on their individual observations and features. Agents are then dynamically assigned to groups based on their representations using k-means clustering. 
Furthermore, the bi-level interaction module improves the traditional MF method by incorporating both inter- and intra-group interactions. Within groups, the classic MF can effectively model agent interactions, given their similarity. However, across groups, the distinct characteristics of agents require the introduction of a group attention mechanism to model inter-group dynamics. The intra-group features are captured by MF, while the inter-group interactions are modeled using the attention mechanism. By combining these two modules, BMF reduces the noise caused by iterative aggregation processes and lowers computational overhead. 
Figure~\ref{BMF_MF} demonstrates the two-phase implementation of BMF, and highlights the differences between normal and MF methods for large-scale MARL, where the normal method has a very high feature dimensionality and the MF method suffers from aggregation noise brought by MF approximation. 
Our main contributions can be summarized as follows:

\begin{enumerate}[label=(\arabic*)]
  \item We develop BMF, a method that involves a bi-level aggregation process to Improve classic MF under large-scale MARL. Besides, we theoretically analyze the validity of BMF.

  \item We propose a dynamic group assignment module using learnable VAE-based agent representations, allowing for adaptive group assignments over time.

  \item We incorporate a bi-level interaction module, combining intra-group mean-field modeling with inter-group attention to reduce aggregation noise in MF.

  \item Experimental results across various large-scale MARL tasks show that BMF offers superior performance while maintaining a moderate computational cost.
\end{enumerate}

\begin{figure}[!h]
  \centering
  \includegraphics[width=1\linewidth]{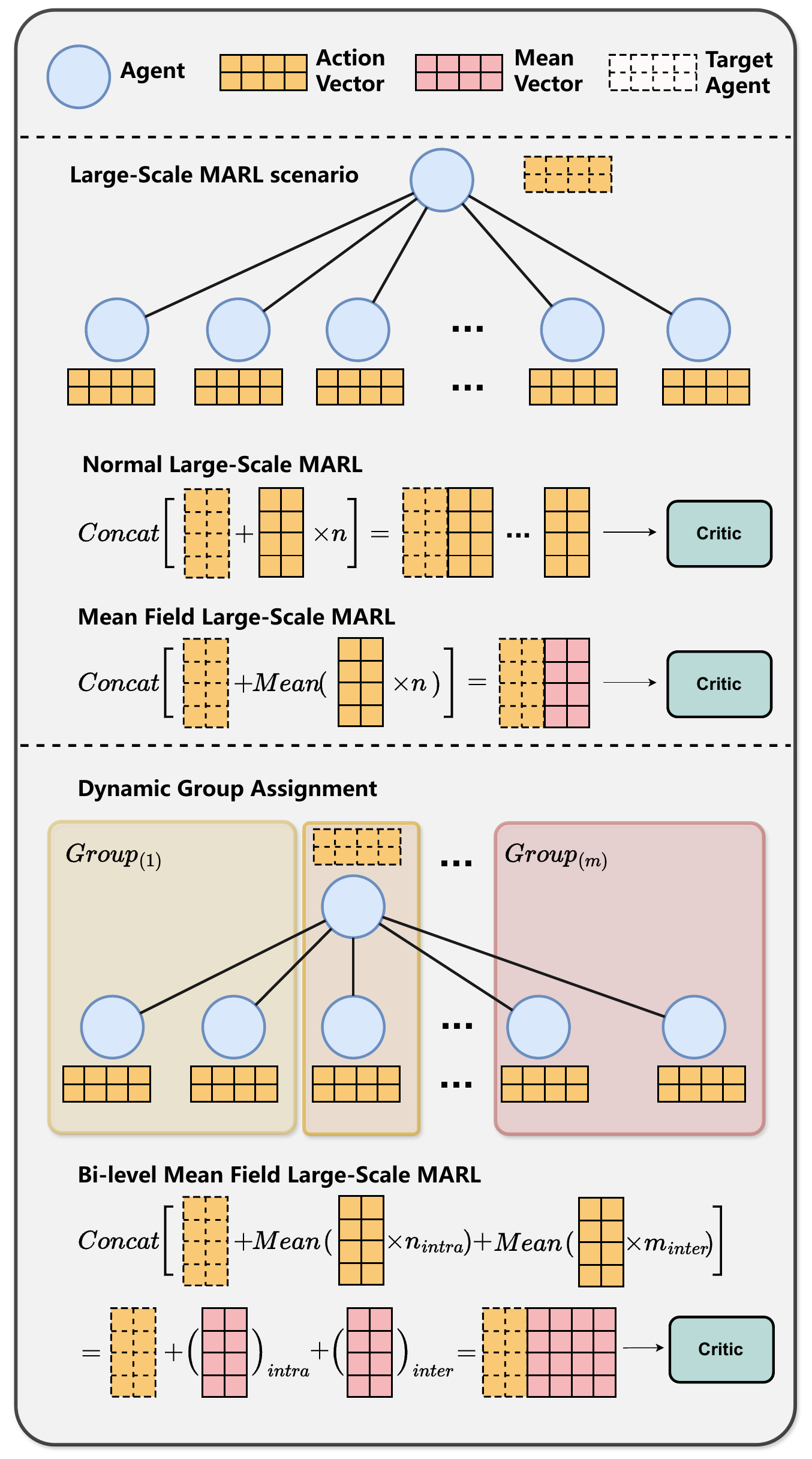}
  \caption{Normal, mean field and bi-level mean field methods apply to the large-scale MARL scenario. Limitation of normal large-scale MARL is the excessively high dimension of the concatenated critic input. Limitation of mean field is ignoring agents' difference and solely using average aggregation, resulting in a loss of precision. Our proposed Bi-level Mean Field (BMF) effectively reduces the critic input dimension while ensuring precision.}
  \label{BMF_MF}
\end{figure} 


\section{Related work}

\textbf{Multi-agent reinforcement learning.}  
Multi-agent reinforcement learning (MARL) has gained significant attention, particularly through the centralized training with decentralized execution (CTDE) framework~\cite{lowe2017multi,rashid2020monotonic,CADP}. 
Several methods exemplify the CTDE framework, including value decomposition (VD) methods and policy-based methods. VD factorizes the joint Q-function into components tied to each agent's local Q-function. VDN~\cite{sunehag2017value} adopts a straightforward additive method for this decomposition, while QMIX~\cite{rashid2020monotonic} uses a monotonic function to optimize joint action values. Building on these concepts, QTRAN~\cite{son2019qtran} and WQMIX~\cite{WQMIX} offer more flexible functional representations, further extending the capabilities of value decomposition techniques. Policy-based methods, such as MAPPO~\cite{yu2022surprising} extends Proximal Policy Optimization (PPO) to multi-agent settings, achieving a balance between sample efficiency and stability. MADDPG~\cite{lowe2017multi} integrates the Actor-Critic architecture with centralized training, enabling agents to share information effectively and handle continuous action spaces. Similarly, COMA~\cite{foerster2018counterfactual} introduces a counterfactual baseline to address the credit assignment problem~\cite{rashid2020monotonic} in cooperative scenarios, improving both policy learning and optimization.
Despite these advances, the scalability of MARL methods remains a significant challenge, particularly in environments with numerous agents. 
As the number of agents increases, the complexity of interactions can lead to computational bottlenecks. In response to this challenge, our proposed method is specifically designed for large-scale multi-agent systems, ensuring robust performance even in complex and dynamic environments.

\textbf{Large-scale MARL.} Mean Field (MF)~\cite{yang2018mean} is the most classic method to reduce the unaffordable high dimension of the centralized training~(CT) features and implement MARL on more agents. The MF models multi-agent interactions as a pairwise interaction between two agents, where one is the specific agent and the other is constructed as a mean field virtual agent. Specifically, the virtual agent corresponds to the mean effect of all the neighboring agents. MF techniques have proven effective in large-scale MARL applications, such as edge computing~\cite{abouaomar2021mean}, game playing~\cite{perrin2021mean} and robotic controlling~\cite{said2021multi}. However, this approximation overlooks the differing strengths of interactions among agents, resulting in precision loss when modeling their complex relationships. To address this limitation, enhancements like MTMF~\cite{ganapathi2020multi} and GAT-MF~\cite{hao2023gat} have been introduced. MTMF categorizes agents into several types and perform mean field approximation for each type, but it still suffers from precision issues and requires prior knowledge for categorization. GAT-MF utilizes parameters to represent agent correlations and performs a weighted mean field approximation, but the virtual mean field agent can only be approximated from fixed neighbors, failing to adapt to dynamic scenarios, especially systems with fluctuating agent numbers. Our method, on the other hand, designs an innovative dynamic grouping attention mechanism for effective approximation while adapting to dynamic scenarios.

\textbf{Agent grouping in large scale MARL.} Agent grouping effectively simplifies large-scale multi-agent problems into smaller, more manageable ones. There are some direct agent grouping methods that conform to human intuition, which are widely used in multi-agent systems to simplify complex problems. MTMF~\cite{ganapathi2020multi} use K-means to group agents based on their observation and state information. SOG~\cite{shao2022self} selects group leaders and clusters agents within their observation range, but the effectiveness of grouping largely hinges on the reasonable selection of leaders. 
DHCG~\cite{liu2023deep} models agent relationships as a cooperation graph and groups agents with a threshold. HCGL~\cite{fu2024self} and GACG~\cite{duan2024group} uses deep neural networks to automatically classify agents into predefined groups. These methods only use agents' original partial information for grouping, which causes noise and is incomplete in complex multi-agent environments. Our method presents an improved classification scheme that leverages agent observations, states, and actions to derive their representations. By encoding agents' randomness into latent variables, BMF effectively reduces noise and enables a reliable agents' classification.

\section{Preliminary}
\subsection{Multi-agent Markov Decision Process}
We discuss the large-scale MARL setting as Multi-agent Markov Decision Process~(MMDP), defined as a tuple $\langle \mathcal{N},\mathcal{S},p_0,\mathcal{A},P,r,\gamma \rangle$, where $\mathcal{N} = \{i\}_{i=1}^n$ is the set consists of $n$ agents and $\boldsymbol{s}=(s_1,\cdots,s_n)\in\mathcal{S}$ is the global state of the environment. The initial state distribution is given by $p_0=\Delta(\mathcal{S})$, which is a probability distribution collection obtained from state space $\mathcal{S}$. The joint action space  $\boldsymbol{a}=(a_1,\cdots,a_n)\in\mathcal{A}= \mathcal{A}_1 \times \cdots \times \mathcal{A}_N$ and each agent's action $a_i \in \mathcal{A}_i$ is produced by policy $\pi_i(a_i  | s)$, forming the joint action $\boldsymbol{a}$ at each time step. The state transition from $s$ to $s'$ for one step is defined by the state transition function $P(s' | s, \boldsymbol{a}):\mathcal{S}\times\mathcal{A}\times\mathcal{S} \to [0,1]$ and will receive a reward via reward function $r(s, \boldsymbol{a}):\mathcal{S}\times\mathcal{A}\to \mathbb{R}$. The long-term reward from $t_0$ to $t$ is defined as: 
\begin{eqnarray}
R = \sum_{t=t_0}^{T}{\gamma^{t-t_0} r_{t}},
\end{eqnarray}
 where $\gamma\in [0, 1]$ is the discount factor and $T$ is the maximum count of steps in one episode.
 
\subsection{Q Learning}
One of the basic methods is Deep Q Learning, which attempts to find an efficient policy by maximizing the value function $Q(s,a)$. $Q(s,a)=\mathbb{E}[r_{t+1}+\gamma r_{t+2}+\gamma^2 r_{t+3}+\cdots|s_t=s,a_t=a]$. For one step transition from state-action unit $(s,a)$ to $(s',a')$, 
 This value function is modeled by minimizing the following loss function:
\begin{equation}
\begin{split}
\mathcal{L} &= \mathbb{E}_{(s,a,r,s')}[(Q(s,a)-y)^2], \\
y &= r+\gamma\mathop{\max}\limits_{a'}Q'_{(s',a')}, \\
\end{split}
\end{equation}
where $Q'$ is synchronized from $Q$ with a time step delay.

\subsection{Policy Gradient}
Policy Gradient addresses continuous action space problems, with Actor-Critic being a classic method. It involves two stages: the policy network outputs an action, and the value network evaluates a Q-value for the state-action pair $(s,a)$. The value network is trained similarly to Q~ Learning:
\begin{eqnarray}
\mathcal{L}=\mathbb{E}_{(s,a,r,s')}[(Q(s,a)-y)^2],\quad y=r+\gamma Q'(s',a'),
\end{eqnarray}
and the policy network is optimized to maximize the value, following the gradient:
\begin{eqnarray}
\nabla J = \mathbb{E}_{a\sim\pi}[\nabla \log\pi(a|s)Q(s,a)],
\end{eqnarray}
where $\pi$  is the policy network. The actions are decided by the policy network, which is well modeled through the Actor-Critic algorithm.


\begin{figure*}[htbp]
  \centering
  \includegraphics[width=1\linewidth]{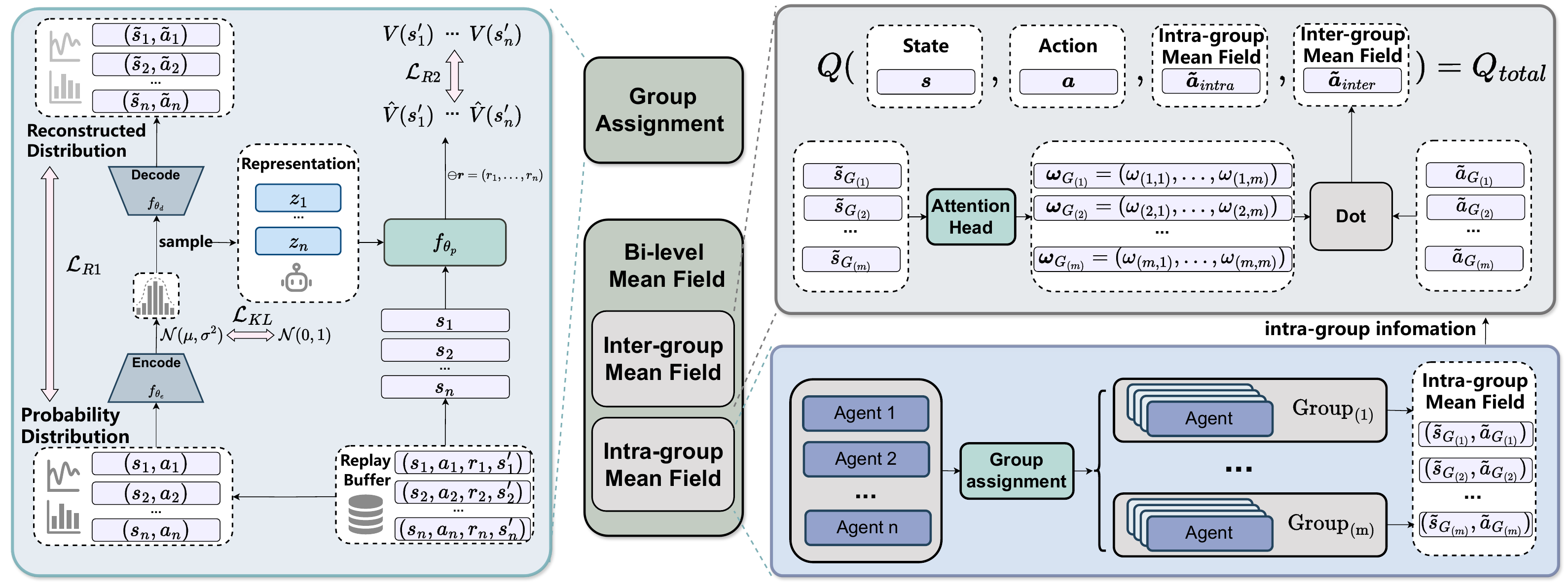}
  \vspace{1pt}
  \caption{The Bi-level Mean Field (BMF) framework. The group assignment forward model for learning agent representations is based on a VAE architecture. The intra-group MF considers the influence of same-type agents and approximates the average effect for a group. The inter-group MF focuses on the differences between different groups and calculates the total Q value.}
  \vspace{6pt}
  \label{BMF_framework}
\end{figure*}


\section{Method}


\subsection{Overview}
As shown in Figure \ref{BMF_MF}, original MF uses an unweighted mechanism to collect the average global information, the approximation of the value function for agent $j$ is considered as:
\begin{eqnarray}
Q_j(s,\boldsymbol{a}) \sim \widetilde{Q}_j(s, a_j, \tilde{a}_j), \quad \tilde{a}_j= \frac{1}{|N(j)|}\sum_{k \in N(j)}a_k,
\end{eqnarray}
where $N(j)$ is the neighboring agents of agent $j$ and the action set $\boldsymbol{a}$ in value function $Q^j$ is approximate by a two actions unit $(a_j,a_k)$ . This method reduces the dimension of $\boldsymbol{a}$ by averaging the neighbor information. However, this unweighted approximation fails to consider that interactions between agents differ across various pairs and can change over time. Therefore, we propose a bi-level mean field with a dynamic group mechanism to refine the types of agents and effectively model agent groups' interactions. The value function is as follow when considering agent $j$ in group $m\in G$, where $G$ is the set of groups and $G(m)$ means the collection of groups in $G$ excluding $m$:
\begin{equation}
\begin{split}
Q_j(s,\boldsymbol{a}) &\sim \widehat{Q}_j(s, a_j, \tilde{a}_j, \tilde{a}_m), 
\quad \tilde{a}_j= \frac{1}{|N_m(j)|}\sum_{k \in N_m(j)}a_k, \\
\tilde{a}_m &= \frac{1}{W_m}\sum_{n \in G(m)}w_{mn}\tilde{a}_n, 
\quad W_m=\sum_{n \in G(m)}w_{mn}.
\end{split}
\end{equation}
Here $N_m(j)$ is the neighboring agent of agent $j$ in group $m$, and $\tilde{a}_j$ is calculated as intra-group mean field action, while $\tilde{a}_m$ is a weighted aggregation based on the inter-group mean field results.

\subsection{Dynamic Group Assignment}
The dynamic group assignment mechanism effectively aggregates similar types of agents, establishing the foundational framework for our bi-level mean field method and reducing the interaction costs among agents. Considering the high randomness in large-scale MARL, the observations and actions of each agent exhibit differences over time, leading to frequent changes in agent types and corresponding adjustments in agent group assignment. Therefore, our method dynamically evaluates agent types and performs grouping at certain time intervals, allowing the group assignment mechanism could adapt to complex large-scale MARL scenarios without requiring prior knowledge to specify agent types.

Our dynamic group assignment extracts representations of agents by inputting their states, observations, and actions information. In order to make sure agents of the same type are focused on actions that own similar effects, we design a mechanism for extracting preference representations that emphasize actions, expecting these representations to be used to derive the reward with the given current state. Additionally, since agents executing the same action are distributed sparsely in large-scale MARL scenarios, we use a variational autoencoder (VAE) to extract the agents' representations. The VAE encodes the randomness of the agent distribution into the representation, serving to filter out noise.

As shown in Figure \ref{BMF_framework}, we employed a mixed model of prediction and regression to learn the agent encoder. For a given agent j, we can obtain the representation $z_j$ and design a reconstruct loss. We also combine $z_j$ and $s_j$ to predict the value of the next state $V(s_j')$, designing a predict loss.
Combining prediction and regression, the loss function is designed as:
\begin{eqnarray}
\begin{aligned}
\mathcal{L}_e(\theta_e,\theta_d,\theta_p)
&=\mathbb{E}_{(\boldsymbol{s},\boldsymbol{a},\boldsymbol{r},\boldsymbol{s'})\thicksim\mathcal{D}}
\Big[\lambda_p\sum_{i=1}^n(d(\boldsymbol{z}_i)-(s_i,a_i))^2 \\
& \quad +\lambda_e\sum_{i=1}^n(p(\boldsymbol{z}_i,s_i)+r_i-\gamma_{s'})^2\Big],
\end{aligned}
\end{eqnarray}
In this paper, the encoder is trained as a forward model, and we simply use $k$-means clustering based on Euclidean distances between agents' representations, but it can be easily extended to other clustering methods.

\subsection{Bi-level Mean Field}
The bilevel mean field method, based on dynamic group assignment, has designed an interaction framework for agents that incorporates both inter-group and intra-group interactions. This framework not only enhances overall performance but also effectively reduces computational overhead in large-scale MARL. By optimizing collaboration and competition mechanisms among agents, this method achieves more efficient resource utilization and decision-making, demonstrating excellent performance in complex environments.

Intra-group interaction employs an unweighted mean field method, treating agents within the same group as homogeneous (assuming similar behaviors among group members). This retains the advantage of low computational overhead associated with mean field calculations. On the other hand, inter-group interaction utilizes a separate unweighted mean field method, treating agents from different groups as heterogeneous. It incorporates an attention mechanism to account for the differences in interactions among heterogeneous agents, effectively modeling their interactions.

Our method includes two forms tailored for discrete and continuous actions, specifically BMF-Q and BMF-AC. BMF-Q follows the optimization method of Q-learning, suitable for decision-making problems in discrete action spaces, while BMF-AC is based on the Deep Deterministic Policy Gradient (DDPG), suitable for optimizing continuous action spaces. This design allows our method to flexibly address different types of reinforcement learning tasks.
For agent $j$ in group $m$, BMF-Q is trained by minimizing the loss function:
\begin{eqnarray}
\begin{aligned}
\mathcal{L}_q(\theta_j)
&=(Q_{\theta_j}(s_j,a_j,\tilde{a}_j,\tilde{a}_m)-y_j)^2,
\end{aligned}
\label{critic_loss}
\end{eqnarray}
where $y_j=\gamma \cdot V_{\theta_j}(s_j')+r_j,\gamma \in (0,1]$ is the target value of state $s$ based on $V_{\theta_j}(s_j')$. And BMF-AC consists of a value network and a policy network. The value network follows the same setting of BMF-Q and the policy network is trained by the policy gradient:
\begin{equation}
\begin{split}
\nabla_{\phi_j}\mathcal{J}(\phi_j)=E_{a_j\sim \pi_{\phi_j}(s)}[\nabla_{\phi_j}\log\pi_{\phi_j}(s)Q_{\theta_j}(s,a_j,\tilde{a}_j,\tilde{a}_m)], 
\label{actor_loss}
\end{split}
\end{equation}
where $\pi_{\phi_j}$ is the policy explicitly modeled by neural networks with the weights $\phi$. BMF pseudocode is provided in Algorithm ~\ref{bmf}.
\begin{algorithm}[!h]
	\caption{Large scale MARL with Bi-level MF} 
	\label{bmf} 
	\begin{algorithmic}
		\REQUIRE Number of agents $n_a$, max episode length $n_e$,  MARL model update interval $I_u$, group assignment interval $I_g$, forward VAE model $M_\theta$, discount factor $\gamma$.
		\ENSURE The trained MARL model.
            \STATE $\theta,\phi \sim$ initial parameters for evaluation network.
            \STATE $\mathcal{D} \leftarrow$ empty replay buffer.
            \FOR{each episode}
                \FOR{$t=1$ to $n_e$}
                    \IF{reach group assignment interval $I_g$}
                        \STATE agent representations $\mathbf{R}_a \leftarrow M_\theta(\boldsymbol{s}_{t-1},\boldsymbol{a}_{t-1})$.
                        \STATE $\mathbf{G} \leftarrow Kmeans(\mathbf{R}_a)$.
                    \ENDIF
                    \FOR{each group $\in \mathbf{G}$}
                        \STATE Calculate each group's representation action: $\tilde{a}_n= \frac{1}{|\mathbf{G}_n|}\sum_{k \in \mathbf{G}_n}a_k$.
                    \ENDFOR
                    \FOR{agent $j=1$ to $n$}
                        \STATE For agent $j$ in group $m$: get intra-group neighboring agents collection $N_m(j)$ and inter-group neighboring agents collection $N_n(j)$.
                        \STATE Calculate the intra-group mean field action: $\tilde{a}_j=\frac{1}{|N_m(j)|}\sum_{k \in N_m(j)}a_k$.
                        \STATE Calculate the inter-group mean field action: $\tilde{a}_m=\frac{1}{W_m}\sum_{n \in \mathbf{G}(m)}w_{mn}\tilde{a}_n$.
                    \ENDFOR
                    \STATE Organize intra-group mean field actions as $\boldsymbol{\tilde{a}_{intra}}$ and inter-group mean field actions as $\boldsymbol{\tilde{a}_{inter}}$.
                    \STATE Calculate and execute the joint action $\boldsymbol{a}_t=(a_1,...,a_n) \in \mathcal{A}$ by $f_\phi$: $\boldsymbol{a}_t \leftarrow f_\phi(\boldsymbol{s}_t)$, get the next state $\boldsymbol{s}_t'=(s_1',...,s_n') \in \mathcal{S}$ and the reward $\boldsymbol{r}_t=(r_1,...,r_n) \in \mathbb{R}^n$.
                    \STATE Store $(\boldsymbol{s},\boldsymbol{a},\boldsymbol{r},\boldsymbol{s'},\boldsymbol{\tilde{a}_{intra}},\boldsymbol{\tilde{a}_{inter}})$ into $\mathcal{D}$.
                    \IF{reach model update interval $I_u$}
                        \STATE Sample an experience $(\boldsymbol{s},\boldsymbol{a},\boldsymbol{r},\boldsymbol{s'},\boldsymbol{\tilde{a}_{intra}},\boldsymbol{\tilde{a}_{inter}})$ from $\mathcal{D}$. 
                        \STATE Update $\theta$ according to critic loss:
                        $\mathcal{L}_q(\theta_j)=(Q_{\theta_j}(s_j,a_j,\tilde{a}_j,\tilde{a}_m)-y_j)^2$.
                        \STATE Update $\phi$ according to actor loss: $\nabla_{\phi_j}\mathcal{J}(\phi_j)=E_{a_j\sim \pi_{\phi_j}(s)}[\nabla_{\phi_j}\log\pi_{\phi_j}(s)Q_{\theta_j}(s,a_j,\tilde{a}_j,\tilde{a}_m)]$.
                    \ENDIF
                \ENDFOR
		\ENDFOR
	\end{algorithmic} 
\end{algorithm}

\subsection{Theoretical Analysis}
We prove the validity of the BMF method under mean field conditions and discuss the error bounds brought by the BMF method. We find that the error of BMF is bounded by a interval $[-2K,2K]$, under the condition that the Q-function is K-smooth. Due to space limitation, we provide the proof validating the BMF here, while the detailed proof of the error bound is provided in the supplementary material.
\begin{assumption}
\label{assumption1}
    The global Q-function is equivalent to the sum of local Q-functions:
    \begin{equation}
    \begin{aligned}
    &Q(\boldsymbol{s},\boldsymbol{a})=\sum_jQ_j(s,\boldsymbol{a}),\\
    \end{aligned}
    \end{equation}
\end{assumption}

\begin{assumption}
\label{assumption2}
    For agent $j$ in group $m$, the local Q-function can be factorized by a set of Q-functions that capture pairwise interactions:
    \begin{equation}
    \begin{aligned}
    Q_j(s,\boldsymbol{a})
    &=\frac{1}{|N_m(j)|}\sum_{k \in N_m(j)}\widetilde{Q}_j(s, a_j, {a}_k) \\
    & \quad +\frac{1}{W_m}\sum_{\substack{n \in G(m)\\k'\in N_n(j)}}w_{mn}\widetilde{Q}_j(s, a_j, {a}_{k'})
    \end{aligned}
    \end{equation}
    where $N_m(j)$ is the set of neighbors of agent $j$ and $n$ is a group different from $m$.
\end{assumption}

\begin{theorem}[BMF approximation]
\label{theorem1}
    When considering agent $j$ in group $m$, the global Q-function can be represented as:
    \begin{equation}
    Q_j(\boldsymbol{s},\boldsymbol{a}) \sim \sum_{m}\widehat{Q}_j(s, a_j, \tilde{a}_j, \tilde{a}_m), 
    \end{equation}
    where $\tilde{a}_j$ is intra-group MF action and $\tilde{a}_m$ is the inter-group MF action.
\end{theorem}

\begin{proof}
    Here we prove the feasibility of using BMF approximation for the global Q-function. 

    As shown in Assumption~\ref{assumption1}, the global Q-function is equivalent to the sum of local Q-functions under MF setting, and we have $Q(\boldsymbol{s},\boldsymbol{a})=\sum_jQ_j(s,\boldsymbol{a})$. Therefore, to demonstrate that Theorem ~\ref{theorem1} is valid, we need to prove the local Q-function $Q_j(s,\boldsymbol{a})$ can be approximated as $\widetilde{Q}_j(s, a_j, \tilde{a}_j, \tilde{a}_m)$.
    
    When considering agent $j$ in group $m$, the value function can be approximated as $Q_j(s,\boldsymbol{a}) \sim \widetilde{Q}_j(s, a_j, \tilde{a}_j, \tilde{a}_m)$. Here we have two types of mean field actions, $\tilde{a}_j$ and $\tilde{a}_m$, respectively for intra-group and inter-group.
    Considering the intra-group mean field action $\tilde{a}_j$, since
    \begin{equation}
    \tilde{a}_j= \frac{1}{|N_m(j)|}\sum_{k \in N_m(j)}a_k,
    \end{equation}
    in group $m$, we can regard each agent $a_k$, $k\in N_m(j)$ has a action represented by $\tilde{a}_j$ with a fluctuation:
    \begin{equation}
    a_k=\tilde{a}_j+\delta a_{jk}, 
    \end{equation}
    and we can get a formula of $\delta a_{jk}$:
    \begin{equation}
    \begin{aligned}
    \frac{1}{|N_m(j)|}\sum_{k \in N_m(j)}\delta a_{jk}
    &=\frac{1}{|N_m(j)|}\sum_{k \in N_m(j)}(a_k-\tilde{a}_j)\\
    &=\frac{1}{|N_m(j)|}\sum_{k \in N_m(j)}a_k-\tilde{a}_j\\
    &=\tilde{a}_j-\tilde{a}_j=0,
    \end{aligned}
    \end{equation}
    Similarly, considering the inter-group mean field action $\tilde{a}_m$, we have
    \begin{equation}
    \tilde{a}_n=\tilde{a}_m+\delta a_{mn}, \quad \frac{1}{W_m}\sum_{n \in G(m)}w_{mn}\delta a_{mn} = 0,
    \end{equation}
    and the isomorphic properties of agents within the same group determine that the agent's action differs from the mean field action of its group by a tiny deviation. We have 
    \begin{equation}
    \begin{aligned}
    &{a}_j \sim \tilde{a}_m, \quad \text{agent } j \in \text{group } m,\\
    &{a}_{k'} \sim \tilde{a}_n, \quad \text{agent } k' \in \text{group } n, \\
    &\quad \quad \quad \quad \quad (m,n \in G, m \neq n)
    \end{aligned}
    \end{equation}
    when considering two agents $j$, $k$, which belong to different groups and have a noticeable difference. The total value function of agent j can be expressed as
    \begin{equation}
    \begin{aligned}
    Q_j(s,\boldsymbol{a}) 
    &=\frac{1}{|N_m(j)|}\sum_{k \in N_m(j)}\widetilde{Q}_j(s, a_j, {a}_k) \\
    & \quad +\frac{1}{W_m}\sum_{\substack{n \in G(m)\\k'\in N_n(j)}}w_{mn}\widetilde{Q}_j(s, a_j, {a}_{k'}) \\
    &=\frac{1}{|N_m(j)|}\sum_{k \in N_m(j)}\widetilde{Q}_j(s, a_j, \tilde{a}_j+\delta a_{jk}) \\
    & \quad +\frac{1}{W_m}\sum_{\substack{n \in G(m)\\k'\in N_n(j)}}w_{mn}\widetilde{Q}_j(s, a_j,\tilde{a}_m + \delta a_{mn}).
    \end{aligned}
    \label{sum_of_Q_j}
    \end{equation}
    Denoting $\widetilde{Q}_j(s, a_j, \tilde{a}_j)$ as $Q_0$, $\widetilde{Q}_j(s, a_j, \tilde{a}_m)$ as $Q_1$, we can obtain a derivation using Taylor's formula:
    \begin{equation}
    \begin{aligned}
    (\ref{sum_of_Q_j})
    &=\frac{1}{|N_m(j)|}\sum_{k \in N_m(j)}\Big[Q_0 + \nabla_{\tilde{a}_j}Q_0 \cdot \delta a_{jk}+ o_{jk} \Big] \\
    & \quad +\frac{1}{W_m}\sum_{\substack{n \in G(m)}}w_{mn}\Big[Q_1 + \nabla_{\tilde{a}_m}Q_1 \cdot \delta a_{mn}+ o_{mn} \Big]
    \\
    &=Q_0 + \nabla_{\tilde{a}_j}Q_0 \cdot \frac{1}{|N_m(j)|}\sum_{k \in N_m(j)}\delta a_{jk}\\
    & \quad +Q_1 + \nabla_{\tilde{a}_m}Q_1 \cdot \frac{1}{W_m}\sum_{\substack{n \in G(m)}}w_{mn}\delta a_{mn} \\
    & \quad +\frac{1}{|N_m(j)|}\sum_{k \in N_m(j)}o_{jk}+\frac{1}{W_m}\sum_{\substack{n \in G(m)}}w_{mn}o_{mn} 
    \\
    &=Q_0 + 0 +Q_1 + 0  \\
    & \quad + \frac{1}{|N_m(j)|}\sum_{k \in N_m(j)}o_{jk}+ \frac{1}{W_m}\sum_{\substack{n \in G(m)}}w_{mn}o_{mn}
    \\
    &\approx Q_0 + Q_1 = \widetilde{Q}_j(s, a_j, \tilde{a}_j)+\widetilde{Q}_j(s, a_j, \tilde{a}_m),
    \end{aligned}
    \label{Q0_Q1}
    \end{equation}
    where $o_jk$ and $o_mn$ denote the Taylor polynomial’s remainder. Since the agents within and between groups do not overlap, the intra-group and inter-group mean field actions can be combined for consideration, forming $\widehat{Q}_j(s, a_j, \tilde{a}_j, \tilde{a}_m) = \widetilde{Q}_j(s, a_j, \tilde{a}_j) + \widetilde{Q}_j(s, a_j, \tilde{a}_m)$. Therefore, we prove $Q_j(s,\boldsymbol{a}) \sim \widehat{Q}_j(s, a_j, \tilde{a}_j, \tilde{a}_m)$ accroding to formula (\ref{sum_of_Q_j}) and (\ref{Q0_Q1}).
\end{proof}

\section{Experiments}

\begin{figure*}[htbp]
	\centering
	\begin{subfigure}{0.98\linewidth}
		\centering
		\includegraphics[width=0.98\linewidth]{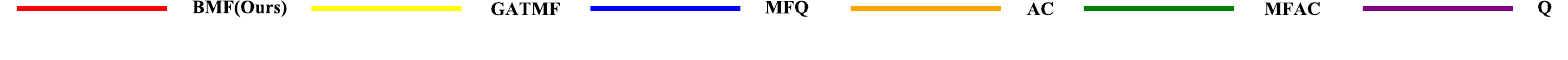}
		\label{FF_legend}
	\end{subfigure}
 
	\centering
	\begin{subfigure}{0.33\linewidth}
		\centering
		\includegraphics[width=0.9\linewidth]{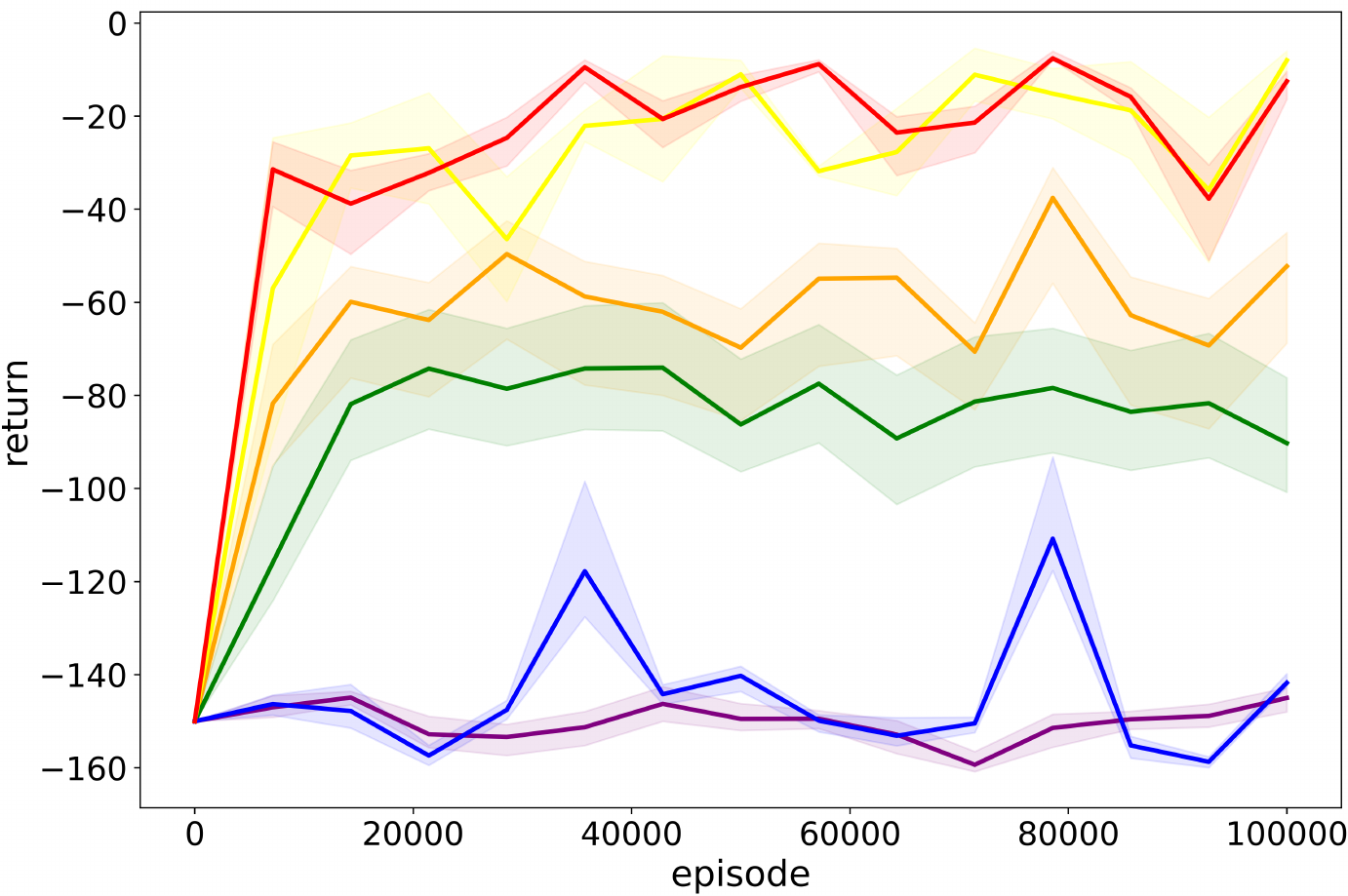}
		\caption{Firefighter}
		\label{FF_per}
	\end{subfigure}
	\centering
	\begin{subfigure}{0.33\linewidth}
		\centering
		\includegraphics[width=0.9\linewidth]{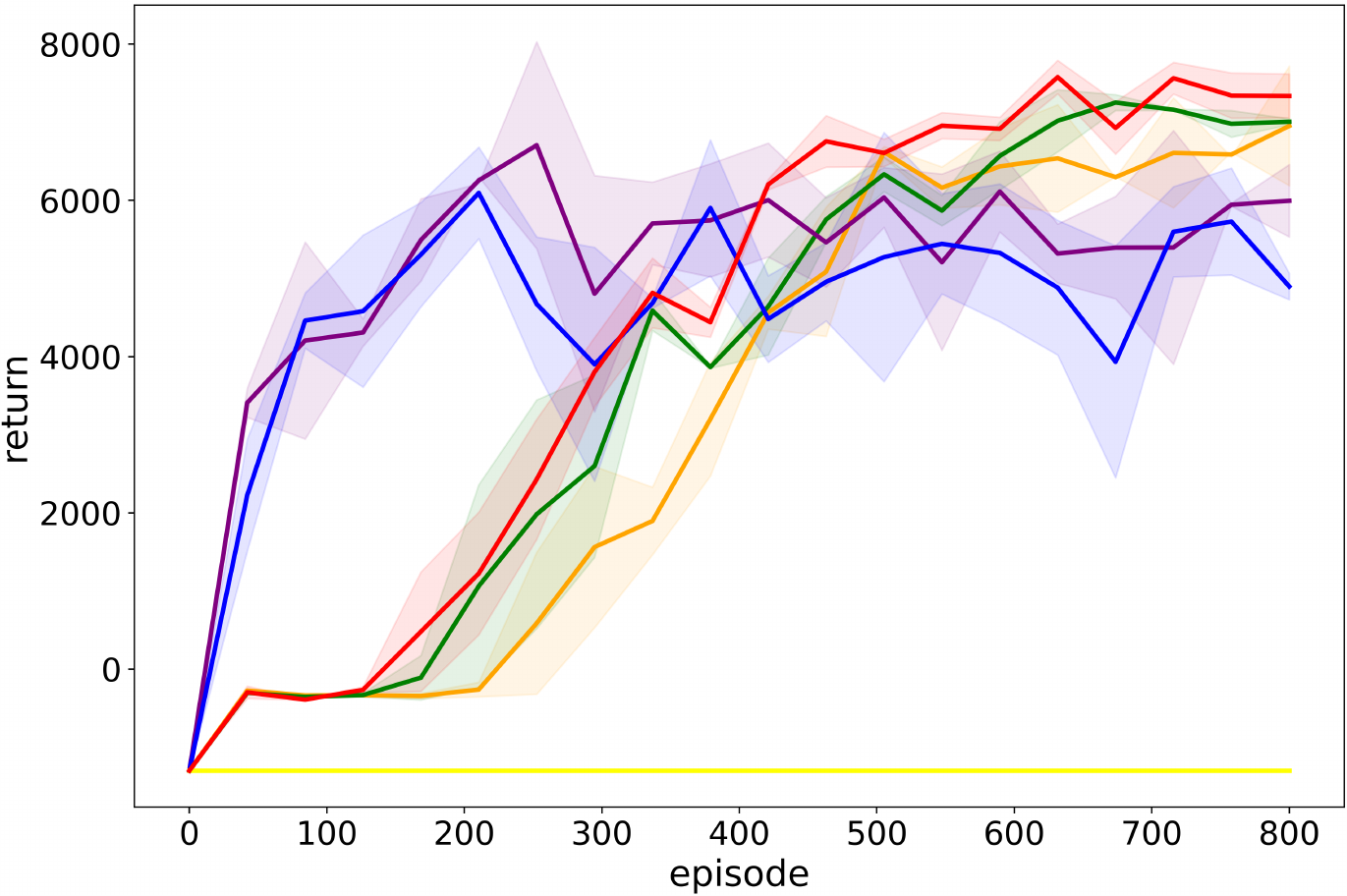}
		\caption{Adversarial Pursuit}
		\label{AP_per}
	\end{subfigure}
	\centering
	\begin{subfigure}{0.33\linewidth}
		\centering
		\includegraphics[width=0.9\linewidth]{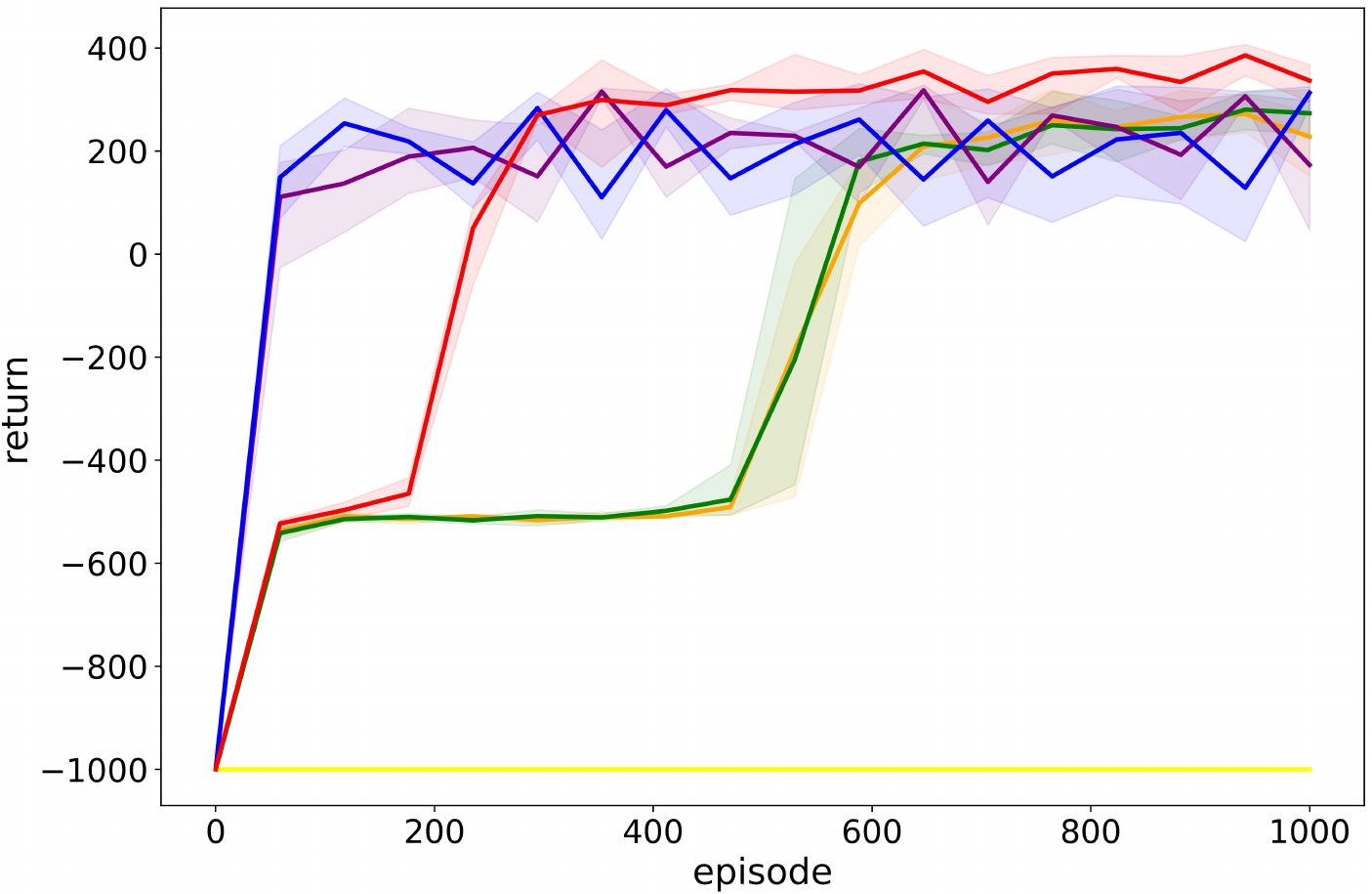}
		\caption{Battle}
		\label{BA_per}
	\end{subfigure}
	\vspace{3pt}
        \caption{Compare BMF with other methods in various large-scale MARL environments. All experimental results are illustrated with the mean and the standard deviation of the metrics over 4 random seeds for a fair comparison. }
        \vspace{3pt}
	\label{Performance}
\end{figure*}

\subsection{Settings}
Our BMF method is tested in three multi-agent experimental environments: \textbf{Firefighter}, \textbf{Adversarial Pursuit}, and \textbf{Battle}, following previous works~\cite{jiang2021multi,zheng2018magent}. Specifically, Firefighter focuses on a fully cooperative task, while the other two focus on mixed cooperative-competitive tasks. The detailed introduction and visualization of the three environments are provided in the supplementary material.
We select the four algorithms tested in the original mean field methods (MFAC, MFQ, AC, Q) as the baseline, and our BMF is also compared with the state-of-the-art method GAT-MF.
Unlike these methods, BMF sets the number of clusters as a hyperparameter and conducts robustness experiments on this hyperparameter.
To further explain the learned collaboration and competition strategies from BMF, we implement a qualitative visualization analysis in \textbf{Battle} environment. The detailed results and analysis is provided in the supplementary material.

\subsection{Performance}
We first compare the performance of our BMF and other baselines under three environments in cooperative and competitive tasks. Figure~\ref{Performance} presents the return curves of comparison methods. 

The \textbf{Firefighter} environment is set up with $50$ firefighter agents and $50$ burning houses, where each agent has a fixed position. Each house initially has a value of $-3$ representing the current fire intensity and decreases by $-1$ at each time step. The total sum of the fire conditions of all houses serves as the final return value. Our BMF method significantly outperforms the MF method and slightly surpasses GAT-MF, which performs excellently in convergence speed and the value obtained.

The \textbf{Adversarial Pursuit} environment is set up with $25$ pursuit agents and $50$ targets. Whenever a pursuit agent successfully tags a target, it receives a reward of $1$. Regardless of whether the tagging is successful or not, there is a penalty of $-0.2$ for each tagging attempt. The total value obtained from successfully tagging all targets is considered the final value. Our BMF method is adapted to this dynamically changing environment compared to GAT-MF, and it has achieved better final values compared to other methods.

The \textbf{Battle} environment is set up as a combat environment with $128$ agents. The agents are evenly divided into two teams that compete against each other. Agents receive small rewards for specific actions, and killing an enemy grants the team a value of $5$. The overall value obtained by the team is considered as the final value. Our BMF method performs better than other methods in this large-scale cooperative competition environment with dynamic changes. 

The results show that BMF can effectively adapt to dynamic and variable large-scale MARL environments, yielding favorable outcomes. 
In addition to evaluating the return value, we also tested the success rate of BMF and the effectiveness of killing enemies in the \textbf{Battle} environment, where agents on the same team must collaborate to compete against opposing agents. The results are provided in the supplementary material.



\subsection{Computational Efficiency}
We compared the time and space costs of BMF with the current state-of-the-art method, GAT-MF. This computational efficiency experiment is conducted in the \textbf{Firefighter} environment and used an NVIDIA RTX A6000 GPU. The results are presented in Table ~\ref{table:computational_efficiency}, which show that BMF reduces the time cost by 31.3\% and the space cost by 15.9\%.

\begin{table}[htbp]
    \centering
    \vspace{3pt}
    \caption{The computational efficiency results with 3 random seeds.}
    \begin{tabular}{ccc}
        \toprule
        method  &  Time(s) & Space(MiB)  \\
        \midrule
        GAT-MF         & 27162±1421.8     &   2116±0.0 \\
        BMF(ours)      & 18655±715.3      &   1780±0.0 \\
        \bottomrule
    \end{tabular}
    \label{table:computational_efficiency}
\end{table}

\subsection{Zero-Shot Generalization}
In this section, we conducted experiments on the battle task with variable agent scales, where the number of agents starts from 128 and increases. The rewards are provided in the table below, which demonstrates that BMF has zero-shot generalization ability across different scales.
\begin{table}[!h]
\label{tab1}
    \vspace{3pt}
    \caption{Zero-shot generalization results with 3 random seeds.}
\resizebox{0.48\textwidth}{!}{ 
\LARGE
\begin{tabular}{cccccc}
\toprule
   \textbf{Num}  & \textbf{Q}  & \textbf{MFQ}  & \textbf{AC}  & \textbf{MFAC}  & \textbf{BMF(ours)} \\
     \midrule
      \textbf{128}   & 186.98 $\pm$ 66.44 & 174.21 $\pm$ 71.42 & 306.19 $\pm$ 23.09 & 318.20 $\pm$ 29.01   & \textbf{396.76 $\pm$ 14.57} \\
      \textbf{288}   & 137.54 $\pm$ 306.07 & 91.91 $\pm$ 108.44  & 565.83 $\pm$ 181.58 & 550.96 $\pm$ 68.01  & \textbf{613.19 $\pm$ 169.70}\\
      \textbf{512}   & 428.37 $\pm$ 542.01 & 176.35 $\pm$ 218.24  & 963.75 $\pm$ 373.27 & 1018.89 $\pm$ 144.66  & \textbf{1188.86 $\pm$ 348.37}\\
      \bottomrule
\end{tabular}}%
    \vspace{-0.35cm}
\end{table}

\subsection{Ablation study}

To analyze the impact of different frameworks and different number of clusters in group assignment, we design ablation experiments to quantitatively validate these effectiveness under the \textbf{Battle} environment. The comparison results of the ablation experiments are presented in Figure~\ref{Abl}. 

\textbf{The ``delays in agent interactions" issue.} The original MFRL uses the action from the previous step to calculate the mean field action, resulting in a time delay. BMF follows the MFRL framework and uses historical mean actions to compute current actions. We implement a no-delay version BMF (follows MTMF's delay handling) and conduct experiments on the battle task with three random seeds. The final rewards are 375.03$\pm$25.44 and 370.81$\pm$35.96, which perform a difference of 1.1\%. Therefore, one-step delay causes slight impact.

\textbf{The robustness of sub-groups.} 
We have additionally conducted experiments to prove that the impact brought by the number of clusters has a certain threshold. When the value of k(the number of sub-groups) rises above this threshold, the variations in k have little impact on the results, demonstrating robustness. We have conducted ablation, and Figure ~\ref{k_abl} shows the performance in the \textbf{Battle} environment under different values of k. Appropriate predefined number of sub-groups is selected for each task.

\textbf{BMF w/o group assignment module.}
To verify the effectiveness of the dynamic group assignment module, we designed BMF-RC, which represents a version of BMF without the dynamic group assignment module. Results in Figure ~\ref{module_abl} show that BMF-RC is inferior to our VAE-based BMF in terms of convergence speed and final value. 

\textbf{BMF with AE-based representation.}
To validate the effectiveness of VAE as the framework for the dynamic group assignment module, we designed BMF-AE, which represents BMF with the AE framework replacing the VAE framework. Results in Figure ~\ref{module_abl} indicate that AE does not have a positive impact on BMF, due to the AE framework's inability to handle the noise introduced by the agents' randomness. Consequently, AE-based BMF is inferior to our VAE-based BMF in terms of convergence speed and final value.

\textbf{BMF with VAE-based representation.}
Our BMF method employs a VAE-based dynamic group assignment module. BMF-VAE represents our method, which demonstrated better performance according to ablation results in Figure ~\ref{module_abl}. 

\begin{figure}[h]
    \centering
	\begin{subfigure}{0.8\linewidth}
        \centering
        \includegraphics[width=0.8\linewidth]{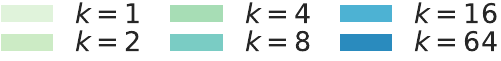}
        
        \includegraphics[width=0.9\linewidth]{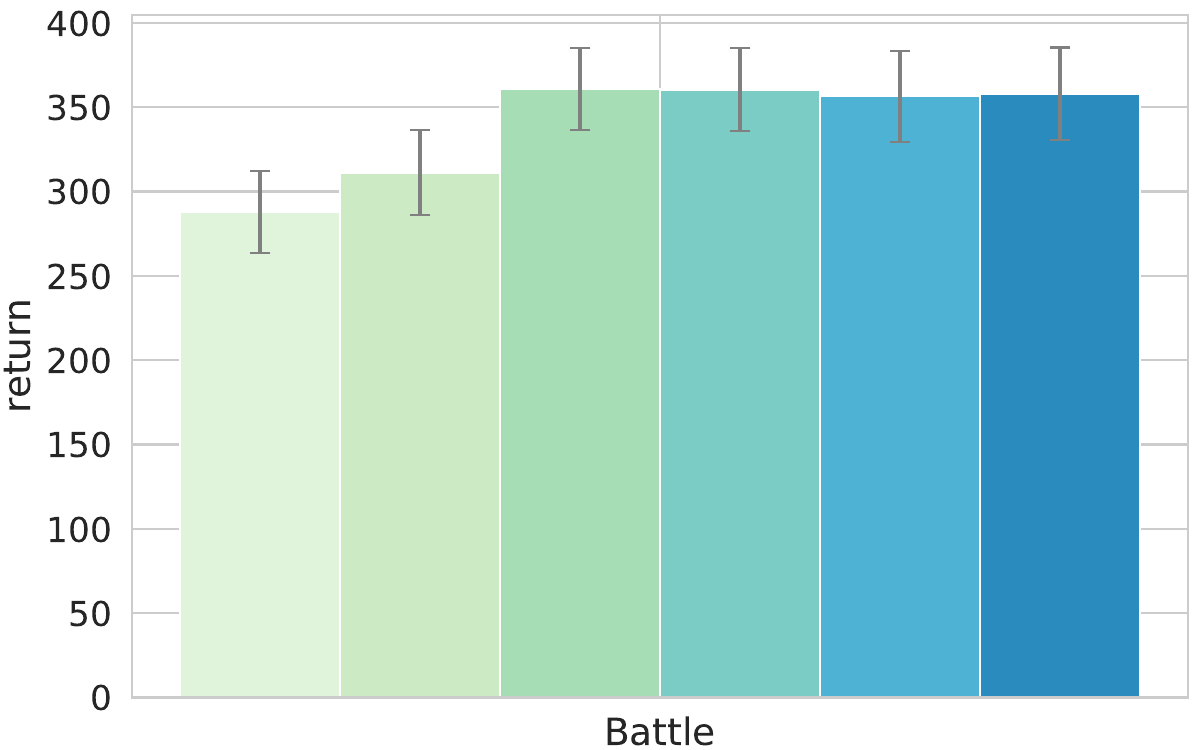}
        \caption{Number of clusters ablation}
        \label{k_abl}
	\end{subfigure}
	\vspace{12pt}
    
    \centering
	\begin{subfigure}{0.8\linewidth}
        \centering
        \includegraphics[width=0.9\linewidth]{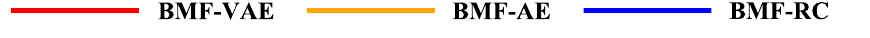}
        
        \includegraphics[width=0.9\linewidth]{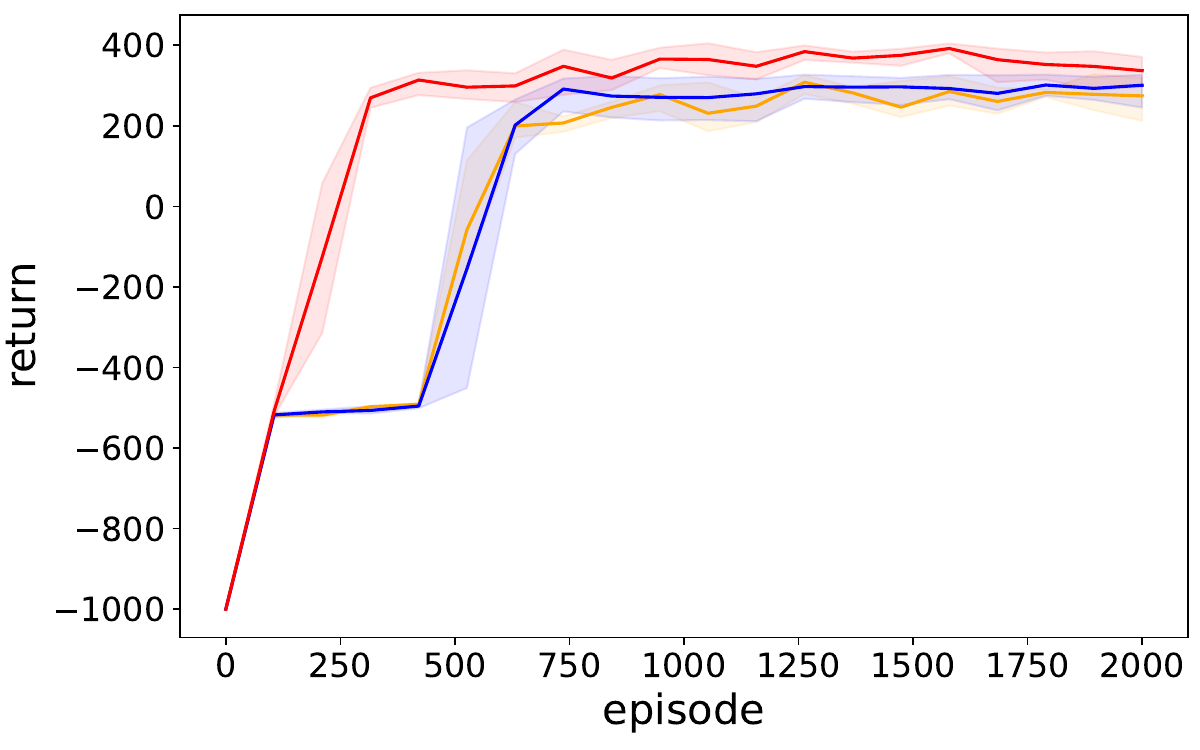}
        \caption{Group assignment ablation}
        \label{module_abl}
	\end{subfigure}
	\vspace{10pt}

	\caption{Ablations about group assignment and the number of clusters under the Battle environment.}
	\vspace{15pt}
	\label{Abl}
\end{figure}

\subsection{Visualization.} 
To further explain the learned collaboration and competition strategies from BMF, we implement a qualitative visualization analysis in \textbf{Battle} environment as shown in Figure~\ref{Vis}. From the behavior slices, it can be observed that the agents exhibit cooperative behaviors such as coordinated attacks and collaborative pursuits. Additionally, agents from different teams engage in competitive behaviors against each other. Our BMF method enables effective cooperation among multiple agents in the same group to chase, while agents in different groups will engage in intense operations such as kite and attack operations with opponents.

\begin{figure}[htbp]
	\centering
        \vspace{15pt}
        \includegraphics[width=1\linewidth]{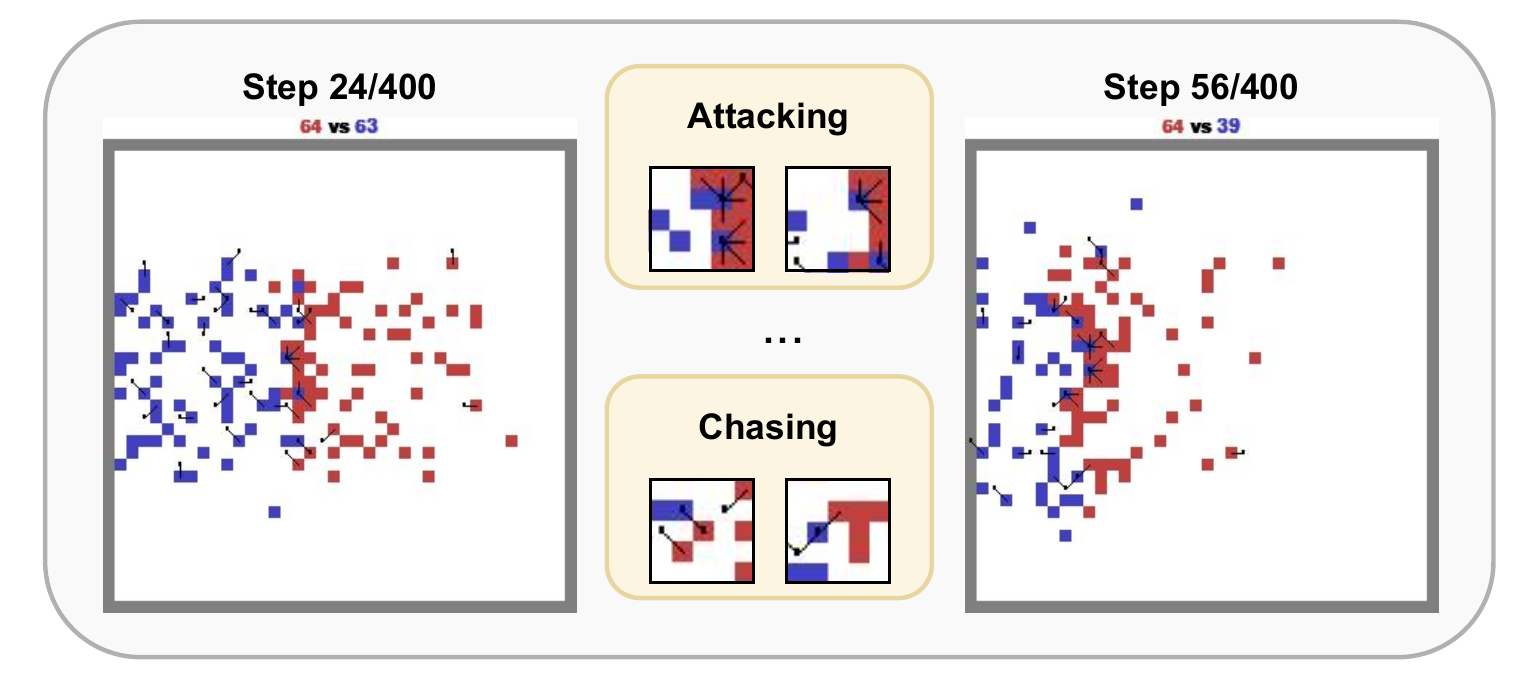}
        \vspace{0pt}
	\caption{Visualization of agents collaboration and competition under the Battle scenario.}
	\label{Vis}
    \vspace{15pt}
\end{figure}

\section{Conclusion}
In this paper, we explore the large-scale MARL problems and propose the Bi-level Mean Field~(BMF) method. Unlike traditional methods, BMF dynamically groups agents based on their extracted hidden features, allowing for a deeper understanding of the relationships between agents. In BMF, we introduce intra-group and inter-group MF. Experiments demonstrate that BMF exhibits strong adaptability across various dynamic large-scale multi-agent environments, outperforming existing methods. 
In future work, we will explore the combination of adaptive grouping mechanisms with BMF and extend BMF to more practical applications.






\appendix
\bibliography{ecai2025}

\end{document}